\documentclass{article}

\usepackage[preprint]{neurips_2026}

\usepackage[utf8]{inputenc} % allow utf-8 input
\usepackage[T1]{fontenc}    % use 8-bit T1 fonts

\usepackage{graphicx}

\usepackage{amsmath, amsfonts}
\usepackage{amsthm, thmtools}

\usepackage{booktabs}       % professional-quality tables
\usepackage{nicefrac}       % compact symbols for 1/2, etc.
\usepackage{microtype}      % microtypography
\usepackage[dvipsnames]{xcolor}         % colors
\usepackage{enumitem}
\usepackage{multirow}

\usepackage[
  colorlinks=true,
  linkcolor=Blue,
  citecolor=Blue,
  urlcolor=Blue
]{hyperref}                 % hyperlinks
\usepackage{url}            % simple URL typesetting
\usepackage{cleveref}       % clever reference

\theoremstyle{plain}
\newtheorem{theorem}{Theorem}
\newtheorem{lemma}[theorem]{Lemma}
\newtheorem{corollary}[theorem]{Corollary}

\theoremstyle{definition}
\newtheorem{definition}{Definition}
\newtheorem{assumption}{Assumption}
\newtheorem{observation}{Observation}

\theoremstyle{remark}
\newtheorem{remark}{Remark}

\crefname{theorem}{theorem}{theorems}
\Crefname{theorem}{Theorem}{Theorems}
\crefname{proposition}{proposition}{propositions}
\Crefname{proposition}{Proposition}{Propositions}
\crefname{lemma}{lemma}{lemmas}
\Crefname{lemma}{Lemma}{Lemmas}
\crefname{corollary}{corollary}{corollaries}
\Crefname{corollary}{Corollary}{Corollaries}
\crefname{definition}{definition}{definitions}
\Crefname{definition}{Definition}{Definitions}
\crefname{remark}{remark}{remarks}
\Crefname{remark}{Remark}{Remarks}
\crefname{assumption}{assumption}{assumptions}
\Crefname{assumption}{Assumption}{Assumptions}
\crefname{observation}{observation}{observations}
\Crefname{observation}{Observation}{Observations}

\title{On the Sparsity of Direct Preference Optimization: Weight Disentanglement in the NTK Regime}

\author{%
  Kensuke Sasaki\\
  Department of Computer Science\\
  The University of Tokyo\\
  \texttt{sasaki-kensuke0222@g.ecc.u-tokyo.ac.jp} \\
  \And
  Issei Sato\\
  Department of Computer Science\\
  The University of Tokyo\\
  \texttt{sato@g.ecc.u-tokyo.ac.jp} \\
}

\begin{document}

\maketitle

\begin{abstract}
  Direct Preference Optimization (DPO) has been widely adopted for aligning language models with human preferences. Recent empirical work has revealed that DPO induces remarkably sparse parameter updates, yet no theoretical explanation for this phenomenon has been offered. We analyze DPO's optimization dynamics under a neural tangent kernel linearization and show that the logistic structure of the DPO loss gives rise to an implicit maximum-margin bias. Building on this observation, we introduce the concept of \emph{dual sparsity}: the DPO update direction is a sparse linear combination (data-level sparsity, driven by support vectors) of sparse gradient feature vectors (parameter-level sparsity, arising from gradient cancellation between preferred and dispreferred responses). We formalize this structure in our main theorem, bound the number of support pairs by a margin-dependent spectral quantity, and present a framework relating dual sparsity to weight disentanglement---a condition favorable for model merging. Experiments on GPT-2 Small and Medium and on 7--9B instruction-tuned models (Llama-3-8B, Mistral-7B, Gemma-2-9B) with the Anthropic HH-RLHF and UltraFeedback datasets confirm both sources of sparsity, and experiments on GPT-2 show that DPO reduces cross-task interference by up to $35{\times}$ compared to supervised fine-tuning.
\end{abstract}

\section{Introduction}
\label{sec:introduction}

Direct Preference Optimization~\citep{rafailov2023direct} has emerged as a dominant paradigm for aligning large language models with human preferences.
By casting the alignment objective as a logistic loss over preference pairs, DPO circumvents the instabilities of reinforcement learning from human feedback while retaining strong empirical performance.
As practitioners increasingly train separate DPO models for distinct objectives---helpfulness, harmlessness, instruction following---there is growing interest in composing these independently trained models via \emph{model merging}~\citep{ilharco2023editing,yu2024language}, which combines task-specific parameter updates (task vectors) into a single multitask model without additional training.

A striking but poorly understood empirical observation is that DPO fine-tuning induces highly sparse parameter updates.
\citet{mukherjee2025sparsity} recently demonstrated that reinforcement learning-based alignment concentrates its modifications onto a small subnetwork of the pre-trained model, leaving the vast majority of parameters essentially unchanged.
This sparsity stands in marked contrast to supervised fine-tuning (SFT), which typically distributes updates broadly across the parameter space.
Independently, \citet{ortizjimenez2023task} identified \emph{weight disentanglement}---the condition that task vectors act on non-overlapping regions of parameter space---as a key property enabling successful model merging.
A natural question arises: is the sparsity of DPO a mere coincidence, or does it reflect a deeper structure in the optimization that is systematically favorable for composing task vectors?

In this paper, we provide a theoretical answer to this question. We analyze the implicit bias of DPO under a neural tangent kernel (NTK) linearization~\citep{jacot2018neural} and show that the logistic loss structure, combined with the geometry of preference-pair gradients, gives rise to a phenomenon we call \emph{dual sparsity}. The DPO update direction converges to the maximum-margin solution in the NTK feature space, inheriting two complementary forms of sparsity: \textbf{data-level sparsity}, where only a small subset of preference pairs (support vectors) contribute to the final update, and \textbf{parameter-level sparsity}, where the gradient feature vector for each pair is itself sparse due to gradient cancellation between preferred and dispreferred responses that share common tokens. Building on this, we present a framework relating dual sparsity to weight disentanglement, and verify empirically that DPO produces cross-task inner products an order of magnitude smaller than those under SFT.

Our main contributions are as follows:
\begin{enumerate}[label=(\textbf{C\arabic*})]
  \item We reduce DPO to maximum-margin classification under NTK linearization and introduce \emph{dual sparsity}---a joint data-level and parameter-level sparsity structure---formalized in \cref{thm:dual_sparsity}. This provides the first theoretical explanation for why DPO induces sparse parameter updates without any explicit sparsity-promoting regularization.

  \item We present a framework relating dual sparsity to weight disentanglement~\citep{ortizjimenez2023task} via a triplet-space formulation of the DPO margin (\cref{sec:wd_connection}). The framework reduces the disentanglement condition to a per-triplet inner product whose magnitude is governed by the parameter-space overlap between tasks---an empirical question that we evaluate in \cref{sec:experiments}.

  \item We validate both sources of sparsity on GPT-2 Small (124M) and Medium (355M) fine-tuned with the Anthropic HH-RLHF~\citep{bai2022training} and UltraFeedback~\citep{cui2024ultrafeedback} datasets, and at modern scale on Llama-3-8B-Instruct, Mistral-7B-Instruct-v0.2, and Gemma-2-9B-it. DPO feature vectors are over 90\% near-zero from 124M to 9B parameters, and the time-integrated weights concentrate on a small subset of pairs within three epochs of full fine-tuning at the 7--8B scale. On GPT-2, cross-task interference under DPO is reduced by $11{\times}$ to $35{\times}$ relative to SFT across all configurations.
\end{enumerate}

The remainder of the paper is organized as follows. \Cref{sec:preliminaries} reviews the DPO formulation and weight disentanglement framework. \Cref{sec:theoretical_analysis} presents our theoretical analysis, culminating in the dual sparsity theorem and a framework connecting it to weight disentanglement. \Cref{sec:experiments} provides experimental validation. \Cref{sec:related_work} discusses related work, and \cref{sec:conclusion} concludes with limitations and future directions.

%% ============================================================

\section{Preliminaries}
\label{sec:preliminaries}

\subsection{Direct Preference Optimization}
\label{sec:dpo}

Let $\pi_{\mathrm{ref}}$ denote a reference policy, typically a pre-trained or supervised fine-tuned (SFT) model~\citep{rafailov2023direct}. Given a prompt $x$ and a pair of responses $(y^+, y^-)$ where $y^+$ is preferred and $y^-$ is dispreferred, the log-probability margin under parameters $\theta$ is
\begin{equation}
  m_i(\theta) = \beta \left( \log \frac{\pi_\theta(y_i^+ \mid x_i)}{\pi_\theta(y_i^- \mid x_i)} - \log \frac{\pi_{\mathrm{ref}}(y_i^+ \mid x_i)}{\pi_{\mathrm{ref}}(y_i^- \mid x_i)} \right),
  \label{eq:margin}
\end{equation}
with temperature $\beta > 0$.
For a preference dataset $\mathcal{D} = \{(x_i, y_i^+, y_i^-)\}_{i=1}^n$, the DPO loss is the logistic loss
\begin{equation}
  \ell_i(\theta) = \log\!\left(1 + e^{-m_i(\theta)}\right),
  \label{eq:dpo_loss}
\end{equation}
whose gradient w.r.t.\ $\theta$ is weighted by $\sigma(-m_i(\theta)) \in (0,1)$: the gradient vanishes for ``easy'' pairs already preferred by the policy, focusing optimization on the ``hard'' subset near the decision boundary.

\subsection{Weight Disentanglement}
\label{sec:weight_disentanglement}

We review the notion of weight disentanglement~\citep{ortizjimenez2023task}, which characterizes an ideal condition for composing independently trained task vectors without cross-task interference.

Consider a neural network $f\colon \mathcal{X} \times \Theta \to \mathcal{Y}$ with pre-trained parameters $\theta_0$. A task $t$ is specified by a data support $\mathcal{D}_t \subseteq \mathcal{X}$, an input distribution $\mu_t$ with $\mathrm{supp}(\mu_t) = \mathcal{D}_t$, and a target function $f_t^*\colon \mathcal{D}_t \to \mathcal{Y}$. Fine-tuning $f$ on task $t$ yields parameters $\theta_t^*$, and the corresponding task vector is defined as $\tau_t := \theta_t^* - \theta_0$.

\begin{definition}[Task Arithmetic, \citealp{ilharco2023editing}]
\label{def:task_arithmetic}
Let $\mathcal{T} = \{\tau_t\}_{t \in [T]}$ be a set of task vectors with mutually disjoint supports $\mathcal{D}_t \cap \mathcal{D}_{t'} = \emptyset$ for $t \neq t'$. The network $f$ satisfies the \emph{task arithmetic property} around $\theta_0$ if
\begin{equation}
  f\!\left(x;\, \theta_0 + \sum_{s=1}^T \alpha_s \tau_s\right) = \begin{cases} f(x;\, \theta_0 + \alpha_t \tau_t) & x \in \mathcal{D}_t, \\ f(x;\, \theta_0) & x \notin \bigcup_{t} \mathcal{D}_t. \end{cases}
  \label{eq:task_arithmetic}
\end{equation}
\end{definition}

\begin{definition}[Weight Disentanglement, \citealp{ortizjimenez2023task}]
\label{def:weight_disentanglement}
The network $f$ is \emph{weight-disentangled} with respect to $\mathcal{T}$ and $\{\mathcal{D}_t\}$ if there exist functions $g_t\colon \mathcal{X} \times \Theta \to \mathcal{Y}$ with $g_t(x; \cdot) = 0$ for $x \notin \mathcal{D}_t$ and a residual $g_0$ with $g_0(x) = 0$ for $x \in \bigcup_t \mathcal{D}_t$, such that
\begin{equation}
  f\!\left(x;\, \theta_0 + \sum_{t=1}^T \alpha_t \tau_t\right) = \sum_{t=1}^T g_t(x;\, \alpha_t \tau_t) + g_0(x).
  \label{eq:weight_disentanglement}
\end{equation}
\end{definition}

It is straightforward to verify that weight disentanglement (\cref{def:weight_disentanglement}) implies the task arithmetic property (\cref{def:task_arithmetic}) by setting $g_t(x; \alpha_t \tau_t) = f(x; \theta_0 + \alpha_t \tau_t)\mathbf{1}(x \in \mathcal{D}_t)$ and $g_0(x) = f(x; \theta_0)\mathbf{1}(x \notin \bigcup_t \mathcal{D}_t)$.

%% ============================================================
\section{Theoretical Analysis}
\label{sec:theoretical_analysis}

Recent empirical work has shown that alignment fine-tuning via DPO induces highly sparse parameter updates relative to the pre-trained weights~\citep{mukherjee2025sparsity}. Independently, weight disentanglement has been identified as a desirable property for robust model merging~\citep{ortizjimenez2023task}. The following questions remain open:
\begin{enumerate}
  \item \textbf{Why does DPO induce sparsity?} What mechanism in the DPO optimization dynamics leads to sparse parameter updates, despite the absence of any explicit sparsity-promoting regularization?
  \item \textbf{What are the implications of this sparsity for cross-task interference?} Specifically, does the sparsity inherent in DPO updates reduce the interference between independently trained task vectors?
\end{enumerate}
This section addresses both questions. We analyze the implicit bias of DPO under a neural tangent kernel (NTK) linearization (\cref{sec:max_margin}) and identify a \emph{dual sparsity} structure that emerges from the interplay between the logistic loss and the geometry of the gradient features (\cref{sec:dual_sparsity}). We then present a framework relating this structure to weight disentanglement, with empirical evaluation deferred to \cref{sec:experiments} (\cref{sec:wd_connection}).

\subsection{DPO as Maximum-Margin Classification in Tangent Space}
\label{sec:max_margin}

We begin by linearizing the DPO margin around the initial parameters $\theta_0$. Under the neural tangent kernel (NTK) regime~\citep{jacot2018neural}, the margin admits a first-order approximation:
\begin{equation}
  m_i(\theta) \approx m_i(\theta_0) + \langle \Phi_i,\, \Delta\theta \rangle,
  \label{eq:margin_linearized}
\end{equation}
where $\Delta\theta = \theta - \theta_0$ and the \emph{gradient feature vector} $\Phi_i$ is defined as
\begin{equation}
  \Phi_i := \beta \left( \nabla_\theta \log \pi_{\theta_0}(y_i^+ \mid x_i) - \nabla_\theta \log \pi_{\theta_0}(y_i^- \mid x_i) \right) \in \mathbb{R}^p.
  \label{eq:feature_vector}
\end{equation}
Here $p$ denotes the total number of parameters. Each $\Phi_i$ encodes the direction in parameter space along which a perturbation $\Delta\theta$ most effectively increases the preference margin for pair $i$. Under this linearization, the DPO loss reduces to a binary logistic regression over the features $\{\Phi_i\}_{i=1}^n$ with offsets $b_i := m_i(\theta_0)$:
\begin{equation}
  \ell_i(\Delta\theta) = \log\!\left(1 + \exp\!\left(-\langle \Phi_i,\, \Delta\theta \rangle - b_i\right)\right).
  \label{eq:logistic_regression}
\end{equation}
With the standard DPO initialization $\pi_{\theta_0} = \pi_{\mathrm{ref}}$, the offsets vanish: $b_i = m_i(\theta_0) = 0$ by \cref{eq:margin}.

We make the following assumption, which is standard in the analysis of implicit bias for logistic regression~\citep{soudry2018implicit}:

\begin{assumption}[Linear Separability]
\label{asm:separability}
There exist $\Delta\theta \in \mathbb{R}^p$ and $\gamma > 0$ such that $\langle \Phi_i,\, \Delta\theta \rangle + b_i \geq \gamma$ for all $i \in [n]$.
\end{assumption}

Under \cref{asm:separability}, the gradient flow on the total loss $L(\Delta\theta) = \sum_i \ell_i(\Delta\theta)$ starting from $\Delta\theta(0) = 0$ admits the explicit trajectory
\begin{equation}
  \Delta\theta(T) = \sum_{i=1}^n A_i(T)\, \Phi_i,
  \quad A_i(T) := \int_0^T \sigma(-m_i(\theta(t)))\, dt,
  \label{eq:trajectory}
\end{equation}
where $A_i(T)$ is the time-integrated gradient weight for pair $i$. Applying the implicit-bias result of \citet{soudry2018implicit} to this setup yields:

\begin{lemma}[Implicit Bias of DPO]
\label{lem:implicit_bias}
Under \cref{asm:separability}, the normalized parameter update converges to the $L_2$ maximum-margin direction:
\begin{equation}
  \frac{\Delta\theta(T)}{\|\Delta\theta(T)\|} \xrightarrow{T \to \infty} u^\star,
  \quad \text{where $u^\star$ solves} \quad
  \max_{\|u\| \leq 1} \min_{i \in [n]} \left\{ \langle u,\, \Phi_i \rangle + b_i \right\}.
  \label{eq:convergence}
\end{equation}
By the Karush--Kuhn--Tucker (KKT) conditions, $u^\star$ admits the representation
\begin{equation}
  u^\star = \sum_{i \in \mathrm{SV}} \alpha_i^\star\, \Phi_i, \qquad \alpha_i^\star \geq 0,
  \label{eq:kkt}
\end{equation}
where $\mathrm{SV} := \{i : \alpha_i^\star > 0\}$ is the set of support pairs.
\end{lemma}

\begin{remark}[Empirical observability of dual variables]
The normalized time-integrated coefficients $\alpha_i(T) := A_i(T) / \|\Delta\theta(T)\|$ also converge to the dual variables $\alpha_i^\star$ as $T \to \infty$. This justifies tracking $A_i$ as an empirical proxy for the support set in \cref{sec:exp_data_sparsity}.
\end{remark}

\subsection{Dual Sparsity of DPO Updates}
\label{sec:dual_sparsity}

\Cref{lem:implicit_bias} establishes that the DPO update direction is fully determined by the support pairs and the corresponding feature vectors. We now identify two distinct sources of sparsity in this representation.

\begin{definition}[Dual Sparsity]
\label{def:dual_sparsity}
The DPO update direction $u^\star = \sum_{i \in \mathrm{SV}} \alpha_i^\star \Phi_i$ is said to exhibit \emph{dual sparsity} if the following two conditions hold:
\begin{enumerate}[label=(\roman*)]
  \item \textbf{Data-level sparsity:} The number of support pairs is small relative to the dataset, i.e., $|\mathrm{SV}| \ll n$.
  \item \textbf{Parameter-level sparsity:} Each gradient feature vector $\Phi_i$ is sparse, i.e., $\|\Phi_i\|_0 \ll p$ for $i \in \mathrm{SV}$.
\end{enumerate}
\end{definition}

The data-level sparsity is a direct consequence of the maximum-margin structure:

\begin{lemma}[Data-Level Sparsity]
\label{lem:data_sparsity}
Suppose $\pi_{\theta_0} = \pi_{\mathrm{ref}}$ (so that $b_i = 0$) and \cref{asm:separability} holds. Let $K := \Phi^\top \Phi \in \mathbb{R}^{n \times n}$, with $K_{ij} = \langle \Phi_i, \Phi_j \rangle$, be the Gram matrix of $\Phi := [\Phi_1 \;\; \Phi_2 \;\; \cdots \;\; \Phi_n]$, and let $\gamma_\star := \min_{i \in [n]} \langle u^\star, \Phi_i \rangle > 0$ be the maximum margin. Then, for the linearized model \cref{eq:logistic_regression},
\begin{equation}
  |\mathrm{SV}|\, \gamma_\star^2 \;\le\; \lambda_{\max}(K_{\mathrm{SV},\mathrm{SV}}) \;\le\; \lambda_{\max}(K),
  \qquad\text{hence}\qquad
  |\mathrm{SV}| \;\le\; \min\left\{ n,\ \left\lfloor \frac{\lambda_{\max}(K)}{\gamma_\star^2} \right\rfloor \right\},
  \label{eq:sv_bound}
\end{equation}
where $K_{\mathrm{SV},\mathrm{SV}}$ is the principal submatrix of $K$ indexed by $\mathrm{SV}$.
\end{lemma}
The proof is given in \cref{app:proof_sv_bound}.

This bound is meaningful because, in the DPO setting, the gradient flow concentrates its cumulative weight $A_i(T)$ on the pairs that remain difficult throughout training---those near the margin boundary---while the remaining pairs receive diminishing gradient signal.
Since the bound is governed by the ratio $\lambda_{\max}(K)/\gamma_\star^2$ rather than explicitly by $n$, the effective support set is provably $\ll n$ rather than merely $\le \mathrm{rank}(\Phi)$ whenever the margin is large relative to the spectral scale of the features; we confirm empirically that this is the case on real preference data (see \cref{sec:exp_data_sparsity}).

For the parameter-level sparsity, we introduce the following assumption:

\begin{assumption}[Sparsity of Gradient Features]
\label{asm:sparse_features}
For each $i \in [n]$, the gradient feature vector $\Phi_i$ satisfies $\|\Phi_i\|_0 \leq s$ for some $s \ll p$.
\end{assumption}

\begin{remark}[Motivation for \cref{asm:sparse_features}]
\label{rem:assumption_motivation}
In autoregressive language models, the preferred and dispreferred responses $y_i^+$ and $y_i^-$ for a given prompt $x_i$ typically share a substantial number of tokens---particularly in the initial portion of the response. Since $\Phi_i$ is defined as the difference of gradients $\beta(\nabla_\theta \log \pi_{\theta_0}(y_i^+ \mid x_i) - \nabla_\theta \log \pi_{\theta_0}(y_i^- \mid x_i))$, the parameter-level contributions from shared tokens cancel, leaving nonzero entries only in the parameters responsible for distinguishing between $y_i^+$ and $y_i^-$. This cancellation effect provides a structural basis for \cref{asm:sparse_features}. We verify this assumption empirically in \cref{sec:experiments}.
\end{remark}

Combining \cref{lem:implicit_bias,lem:data_sparsity} with \cref{asm:sparse_features}, we arrive at the main result:

\begin{theorem}[Dual Sparsity of DPO]
\label{thm:dual_sparsity}
Under \cref{asm:separability,asm:sparse_features}, the DPO update direction $u^\star$ exhibits dual sparsity in the sense of \cref{def:dual_sparsity}. That is, $u^\star$ is a sparse linear combination of sparse vectors:
\begin{equation}
  u^\star = \sum_{i \in \mathrm{SV}} \alpha_i^\star\, \Phi_i, \quad |\mathrm{SV}| \ll n, \quad \|\Phi_i\|_0 \leq s \ll p.
  \label{eq:dual_sparsity}
\end{equation}
Consequently, $u^\star$ is itself sparse with $\|u^\star\|_0 \leq |\mathrm{SV}| \cdot s$.
\end{theorem}

\subsection{Connection to Weight Disentanglement}
\label{sec:wd_connection}

\Cref{thm:dual_sparsity} establishes that DPO updates exhibit dual sparsity. We now examine how this structure relates to the weight disentanglement framework reviewed in \cref{sec:weight_disentanglement}. The original formulation of weight disentanglement (\cref{def:weight_disentanglement}) is stated for a generic function $f\colon \mathcal{X} \times \Theta \to \mathcal{Y}$; applying it to the DPO setting requires identifying the appropriate input space and function in preference learning.

\paragraph{Triplet-space formulation.}
We take the input space to be the space of preference triplets $\tilde{\mathcal{X}} := \{(x, y^+, y^-)\}$ and define $f\colon \tilde{\mathcal{X}} \times \Theta \to \mathbb{R}$ as the DPO log-probability margin
\begin{equation}
  f(\tilde{x}; \theta) := \beta\!\left(\log \pi_\theta(y^+ \mid x) - \log \pi_\theta(y^- \mid x)\right), \quad \tilde{x} = (x, y^+, y^-).
  \label{eq:dpo_margin_function}
\end{equation}
Differentiating at $\theta_0$ recovers the gradient feature vector of \cref{sec:max_margin}: $\nabla_\theta f(\tilde{x}; \theta_0) = \Phi(\tilde{x})$. The same features that drive dual sparsity therefore serve as the Jacobian relevant for weight disentanglement. A task $t$ is associated with a preference dataset $\mathcal{D}_t \subseteq \tilde{\mathcal{X}}$ collected under task-specific annotation criteria.

\paragraph{Linearized decomposition.}
For task vectors $\{\tau_t\}_{t \in [T]}$ and coefficients $\{\lambda_t\}_{t \in [T]}$, a first-order Taylor expansion of $f$ around $\theta_0$ gives
\begin{equation}
  f\!\left(\tilde{x};\, \theta_0 + \sum_{t=1}^{T} \lambda_t \tau_t\right) \;\approx\; f(\tilde{x}; \theta_0) + \sum_{t=1}^{T} \lambda_t \langle \tau_t,\, \Phi(\tilde{x}) \rangle.
  \label{eq:wd_taylor}
\end{equation}
\Cref{eq:wd_taylor} matches the structural form of the weight disentanglement decomposition (\cref{def:weight_disentanglement}) under the identification
\begin{equation}
  g_0(\tilde{x}) := f(\tilde{x}; \theta_0), \qquad g_t(\tilde{x};\, \lambda_t \tau_t) := \lambda_t \langle \tau_t,\, \Phi(\tilde{x}) \rangle.
  \label{eq:wd_identification}
\end{equation}
With this identification, the disentanglement condition $g_t(\tilde{x};\, \lambda_t \tau_t) = 0$ for $\tilde{x} \notin \mathcal{D}_t$ reduces to the inner-product condition
\begin{equation}
  \langle \tau_t,\, \Phi(\tilde{x}) \rangle \;\approx\; 0 \quad \text{for } \tilde{x} \notin \mathcal{D}_t.
  \label{eq:wd_condition}
\end{equation}

\begin{observation}[Linearized weight disentanglement under DPO]
\label{obs:linearized_wd}
Under the NTK linearization, the DPO margin $f$ admits the decomposition \cref{eq:wd_taylor}, and the identification \cref{eq:wd_identification} places this decomposition in the form of weight disentanglement (\cref{def:weight_disentanglement}). The disentanglement condition holds to the extent that the cross-task inner product \cref{eq:wd_condition} is small.
\end{observation}

\paragraph{Connection to dual sparsity.}
\Cref{lem:implicit_bias} establishes that, in the limit, the normalized task vector for task $t$ aligns with the maximum-margin direction $u_t^\star = \sum_{i \in \mathrm{SV}_t} \alpha_{t,i}^\star\, \Phi(\tilde{x}_i^{(t)})$. Substituting into \cref{eq:wd_identification}, the local function $g_t$ takes the form
\begin{equation}
  g_t(\tilde{x};\, \lambda_t \tau_t) \;\propto\; \langle u_t^\star,\, \Phi(\tilde{x}) \rangle \;=\; \Big\langle \sum_{i \in \mathrm{SV}_t} \alpha_{t,i}^\star\, \Phi(\tilde{x}_i^{(t)}),\; \Phi(\tilde{x}) \Big\rangle.
  \label{eq:wd_implicit_bias}
\end{equation}
The cross-task interference \cref{eq:wd_condition} is therefore governed by $\langle u_1^\star,\, \Phi(\tilde{x}) \rangle$ for $\tilde{x} \in \mathcal{D}_2$. By \cref{thm:dual_sparsity}, $u_1^\star$ is a sparse combination of sparse vectors; \cref{asm:sparse_features} ensures $\Phi(\tilde{x})$ is also sparse. When the active parameter supports of $u_1^\star$ and $\{\Phi(\tilde{x})\}_{\tilde{x} \in \mathcal{D}_2}$ have small overlap, this inner product is small. Whether the overlap is in fact small for naturally occurring preference tasks is an empirical question, which we examine in \cref{sec:experiments}.

\paragraph{Remark on the disjoint-support assumption.}
\Cref{def:task_arithmetic} assumes disjoint task supports $\mathcal{D}_t \cap \mathcal{D}_{t'} = \emptyset$. In the triplet-space formulation this is naturally satisfied: distinct preference datasets arise from distinct annotation processes, and an ordered triplet $(x, y^+, y^-)$ is tied to the task it was annotated under.

%% ============================================================

\section{Experiments}
\label{sec:experiments}

We empirically validate the theoretical predictions of \cref{sec:theoretical_analysis}.
Specifically, we verify (i)~data-level sparsity of the dual coefficients $\alpha_i^\star$ (\cref{lem:data_sparsity}), (ii)~parameter-level sparsity of the gradient features $\Phi_i$ (\cref{asm:sparse_features}), and (iii)~suppression of cross-task interference under DPO compared to SFT (\cref{sec:exp_interference}).

\subsection{Experimental Setup}
\label{sec:setup}

\paragraph{Models.}
We use GPT-2 Small and GPT-2 Medium~\citep{radford2019language}, together with three instruction-tuned models with 7--9B parameters: Llama-3-8B-Instruct~\citep{grattafiori2024llama}, Mistral-7B-Instruct-v0.2~\citep{jiang2023mistral}, and Gemma-2-9B-it~\citep{gemmateam2024gemma2}. Each checkpoint serves as both the initialization and the reference policy, $\pi_{\theta_0} = \pi_{\mathrm{ref}}$.

\paragraph{Datasets.}
For our primary experiments, we use the Anthropic HH-RLHF dataset~\citep{bai2022training}, which consists of two subsets: \emph{helpful} (evaluating which response is more useful to the user) and \emph{harmless} (evaluating which response avoids harmful content).
These two distinct subsets serve as the tasks $\mathcal{T}_1$ and $\mathcal{T}_2$ in our cross-task interference analysis.
To assess generality, we additionally use the UltraFeedback Binarized dataset~\citep{cui2024ultrafeedback}, a large-scale preference dataset annotated by GPT-4 across four evaluation axes (instruction-following, truthfulness, honesty, and helpfulness).

\paragraph{Gradient features.}
The gradient feature vectors $\Phi_i = \beta(\nabla_\theta \log \pi_{\theta_0}(y_i^+ \mid x_i) - \nabla_\theta \log \pi_{\theta_0}(y_i^- \mid x_i))$ are computed empirically using the Jacobian of the model at initialization, since the analytical NTK cannot be obtained in closed form for self-attention architectures.
Because storing full-dimensional gradient vectors ($\Phi_i \in \mathbb{R}^p$ with $p > 10^8$) is memory-intensive, experiments that require $\Phi_i$ are performed on subsamples of the datasets (see details in \cref{app:exp_details}).

The models, datasets, and numbers of preference pairs used in each experiment, together with the computation of the time-integrated weights $A_i$, are detailed in \cref{app:exp_details}.

%% ----------------------------------------------------------
\subsection{Data-Level Sparsity}
\label{sec:exp_data_sparsity}

\Cref{lem:data_sparsity} predicts that the DPO update direction is driven by a small number of support pairs, with the remaining pairs contributing negligibly. To verify this, we examine how the time-integrated weights $A_i(T)$ are distributed across pairs as training progresses.

\Cref{fig:data_sparsity} visualizes $A_i$ for GPT-2 Small and Medium on $N = 100$ pairs from HH-RLHF at 1, 100, and 1000 epochs. At early stages, the weights are distributed nearly uniformly. As training progresses, the distribution becomes increasingly polarized: a small subset of pairs accumulates large weights while the rest stop growing (see \cref{app:logps} for the corresponding log-probabilities). This is the behavior predicted by \cref{lem:implicit_bias}, under which the normalized coefficients $\alpha_i(T)$ converge to the dual variables $\alpha_i^\star$, which are nonzero only for the support pairs.

\Cref{tab:data_sparsity} quantifies this concentration on $N = 1000$ pairs by the top-10\% mass, i.e., the fraction of $\sum_i A_i$ held by the top 10\% of pairs (10\% under a uniform distribution). For GPT-2, the top-10\% mass grows from about 11\% to 40--44\% over 1000 epochs; for Llama-3-8B-Instruct and Mistral-7B-Instruct-v0.2, it grows from about 11\% to 44--52\% within three epochs of full fine-tuning DPO.
The numerical evaluation of the bound in \cref{lem:data_sparsity} in the setting of \cref{fig:data_sparsity} is reported in \cref{app:sv_numerical}.

\begin{figure}[t]
  \centering
  \includegraphics[width=\linewidth]{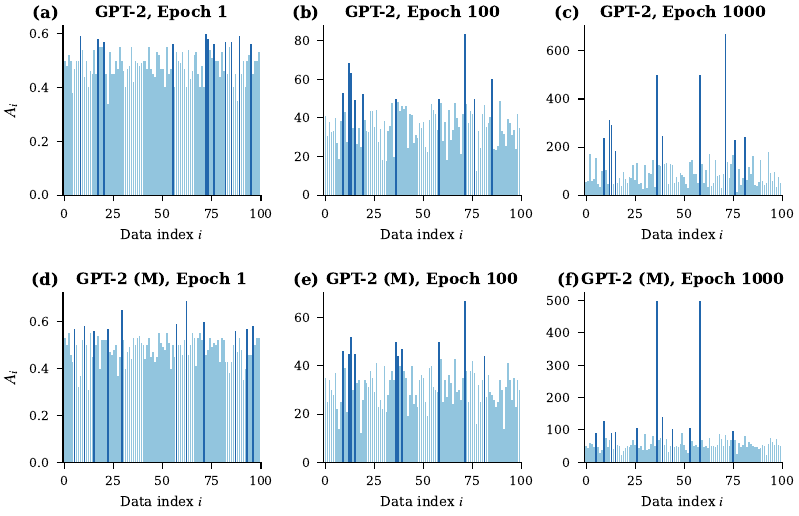}
  \caption{Distribution of $A_i$ across data pairs for GPT-2 Small (top) and Medium (bottom) at 1 epoch (left), 100 epochs (middle), and 1000 epochs (right) on HH-RLHF. Dark bars indicate the top 10\% of pairs by weight. As training progresses, the distribution polarizes onto a small number of support pairs; this effect is more pronounced in the larger model.}
  \label{fig:data_sparsity}
\end{figure}

\begin{table}[t]
  \centering
  \caption{Top-10\% mass of $A_i$ on $N = 1000$ pairs, i.e., the fraction of $\sum_i A_i$ held by the top 10\% of pairs (10\% under a uniform distribution). Top: GPT-2 trained with DPO on HH-RLHF. Bottom: 7--8B models trained with full fine-tuning DPO on UltraFeedback Binarized.}
  \label{tab:data_sparsity}
  \small
  \begin{tabular}{llccc}
    \toprule
    \textbf{Model} & \textbf{Dataset} & \multicolumn{3}{c}{\textbf{Top-10\% mass (\%)}} \\
    \cmidrule(lr){3-5}
      & & Epoch 1 & Epoch 100 & Epoch 1000 \\
    \midrule
    GPT-2 Small  & HH-RLHF & $10.88$ & $28.78$ & $39.57$ \\
    GPT-2 Medium & HH-RLHF & $11.02$ & $32.32$ & $43.70$ \\
    \midrule
      & & Epoch 1 & Epoch 2 & Epoch 3 \\
    \midrule
    Llama-3-8B-Instruct      & UltraFeedback & $11.75$ & $41.82$ & $52.48$ \\
    Mistral-7B-Instruct-v0.2 & UltraFeedback & $11.23$ & $39.92$ & $43.77$ \\
    \bottomrule
  \end{tabular}
\end{table}

%% ----------------------------------------------------------
\subsection{Parameter-Level Sparsity}
\label{sec:exp_param_sparsity}

\Cref{asm:sparse_features} posits that the DPO gradient feature vectors $\Phi_i$ are sparse due to gradient cancellation between preferred and dispreferred responses.
To verify this and to contrast it with SFT, \cref{fig:param_sparsity} visualizes the element-wise distribution of a representative $\Phi_i$ under both methods (GPT-2 Medium, HH-RLHF, $N = 100$), and \cref{tab:param_sparsity} quantifies the fraction of near-zero entries for all five models ($N = 1000$, UltraFeedback Binarized).

Under DPO (\cref{fig:param_sparsity}a), the vast majority of entries are concentrated near zero, with only a small fraction taking on appreciable magnitude.
Under SFT (\cref{fig:param_sparsity}b), the distribution is substantially broader, reflecting the fact that the SFT feature vector $\Phi_i^{\mathrm{SFT}} = \nabla_\theta \log \pi_\theta(y_i^+ \mid x_i)$ involves only the preferred response and does not benefit from the gradient cancellation that arises in the DPO difference $\nabla_\theta \log \pi_{\theta_0}(y_i^+ \mid x_i) - \nabla_\theta \log \pi_{\theta_0}(y_i^- \mid x_i)$.
The same contrast holds at every scale we test (\cref{tab:param_sparsity}): from 124M to 9B parameters, more than 90\% of the entries of the DPO feature vector are near zero, whereas only 4--20\% of those of the SFT feature vector are.

\begin{figure}[t]
  \centering
  \includegraphics[width=\linewidth]{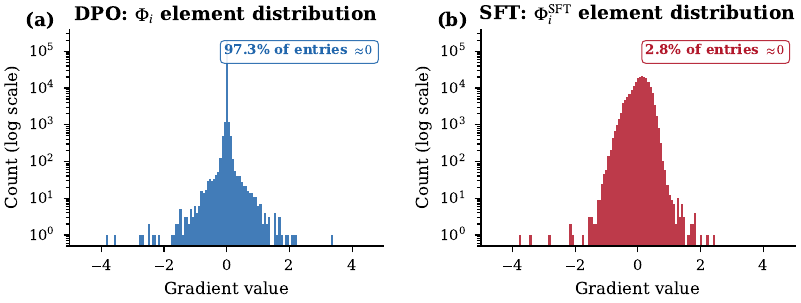}
  \caption{Element-wise distribution of gradient feature vectors on a logarithmic scale for GPT-2 Medium on $N = 100$ pairs from HH-RLHF. \textbf{(a)}~DPO: the vast majority of entries are near zero, confirming \cref{asm:sparse_features}. \textbf{(b)}~SFT: the distribution is substantially denser, since no gradient cancellation occurs. Near-zero fractions for all five models, computed in a separate setting, are reported in \cref{tab:param_sparsity}.}
  \label{fig:param_sparsity}
\end{figure}

As discussed in \cref{rem:assumption_motivation}, this contrast arises naturally from the structure of preference data in autoregressive models: since $y_i^+$ and $y_i^-$ typically share a large number of tokens---particularly in the initial portion of the response---the parameter-level gradients for these shared tokens cancel in the DPO feature vector, leaving nonzero contributions only for the parameters that distinguish between the two responses.

\begin{table}[t]
  \centering
  \caption{Fraction of near-zero entries in a representative gradient feature vector under DPO ($\Phi_i$) and SFT ($\Phi_i^{\mathrm{SFT}}$), computed from 1{,}000 pairs sampled from UltraFeedback Binarized. Higher values indicate sparser features.}
  \label{tab:param_sparsity}
  \small
  \begin{tabular}{lrcc}
    \toprule
    \textbf{Model} & \textbf{Params} & \textbf{DPO $\Phi_i$ (\%)} & \textbf{SFT $\Phi_i^{\mathrm{SFT}}$ (\%)} \\
    \midrule
    GPT-2 Small              & 124M & $91.29$ & $3.86$ \\
    GPT-2 Medium             & 355M & $94.82$ & $4.34$ \\
    Llama-3-8B-Instruct      & 8B   & $92.51$ & $16.44$ \\
    Mistral-7B-Instruct-v0.2 & 7B   & $90.93$ & $19.72$ \\
    Gemma-2-9B-it            & 9B   & $92.12$ & $13.57$ \\
    \bottomrule
  \end{tabular}
\end{table}

%% ----------------------------------------------------------
\subsection{Cross-Task Interference}
\label{sec:exp_interference}

We evaluate whether the dual sparsity of DPO updates translates into reduced cross-task interference, as discussed in \cref{sec:wd_connection}. We consider two task pairs of increasing semantic distance:
\begin{enumerate}[label=(\alph*)]
  \item \textbf{HH helpful $\leftrightarrow$ harmless}: two subsets of the same dataset (Anthropic HH-RLHF) with different annotation criteria---helpfulness vs.\ harmlessness.
  \item \textbf{HH helpful $\leftrightarrow$ UltraFeedback}: two entirely different datasets with different annotation methods (human vs.\ AI), different prompt distributions, and different preference dimensions.
\end{enumerate}
For the cross-dataset pair (b), only the gradient features $\Phi_j^{(2)}$ of UltraFeedback at $\theta_0$ are required; no DPO training on UltraFeedback is performed.

\paragraph{Baseline.}
We compare DPO against SFT. Unlike DPO, SFT does not produce the logistic loss structure that gives rise to implicit maximum-margin bias; consequently, all data pairs contribute to the update direction, and data-level sparsity does not emerge. Moreover, the SFT feature vector $\Phi_i^{\mathrm{SFT}} = \nabla_\theta \log \pi_\theta(y_i^+ \mid x_i)$ involves only the preferred response, so it does not benefit from the gradient cancellation that induces parameter-level sparsity in DPO.

\paragraph{Metrics.}
The disentanglement condition derived in \cref{sec:wd_connection} is a per-triplet condition $\langle \tau_t,\, \Phi(\tilde{x}) \rangle \approx 0$ for $\tilde{x} \notin \mathcal{D}_t$ (\cref{eq:wd_condition}). Verifying it directly would require evaluating $\langle u_1^\star,\, \Phi(\tilde{x}) \rangle$ for every $\tilde{x} \in \mathcal{D}_2$, which is computationally prohibitive given the scale of $\Phi(\tilde{x}) \in \mathbb{R}^p$ with $p > 10^8$. We therefore aggregate the per-triplet quantity into a single scalar summary by introducing the mean feature vector for task $t$:
\begin{equation}
  V_t := \frac{1}{|\mathcal{D}_t|} \sum_{\tilde{x} \in \mathcal{D}_t} \Phi(\tilde{x}).
  \label{eq:mean_feature}
\end{equation}
On this aggregated quantity, we evaluate two measures of cross-task overlap between $V_1$ and $V_2$:
\begin{itemize}
  \item \textbf{Top-1\% Jaccard overlap}: the Jaccard coefficient between the index sets of the top 1\% largest-magnitude entries in each mean feature vector. This captures the degree of overlap between the most active parameters of the two tasks---the structural condition under which the per-triplet inner product $\langle u_1^\star,\, \Phi(\tilde{x}) \rangle$ is small (\cref{sec:wd_connection}).
  \item \textbf{Interference} $\langle u_1^\star,\, V_2 \rangle$: the inner product between the DPO update direction for $\mathcal{T}_1$ and the mean feature vector for $\mathcal{T}_2$. This is the average of $\langle u_1^\star,\, \Phi(\tilde{x}) \rangle$ over $\tilde{x} \in \mathcal{D}_2$ and serves as a tractable summary of the per-triplet disentanglement condition \cref{eq:wd_condition}.
\end{itemize}

\begin{table}[t]
  \centering
  \caption{
    Cross-task interference for DPO and SFT across two model scales (GPT-2 Small/Medium) and two task pairs:
    \emph{HH help$\leftrightarrow$harm} (HH-RLHF helpful vs.\ harmless subsets) and \emph{HH help$\leftrightarrow$UF} (HH-RLHF helpful vs.\ UltraFeedback).
    Lower values indicate greater disentanglement.
    Across all configurations, DPO reduces the interference $\langle u_1^\star, V_2 \rangle$ by $11\times$ to $35\times$ relative to SFT.
  }
  \label{tab:interference}
  \small
  \begin{tabular}{llccc}
    \toprule
    \textbf{Task Pair} & \textbf{Method} & \textbf{Model} & \textbf{Jaccard (\%)} & $\boldsymbol{\langle u_1^\star, V_2 \rangle}$ \\
    \midrule
    \multirow{4}{*}{HH\,help$\leftrightarrow$harm}
      & SFT & Small  & $54.9$ & $9.12$ \\
      & DPO & Small  & $18.9$ & $0.258$ \\
      & SFT & Medium & $49.3$ & $3.22$ \\
      & DPO & Medium & $11.4$ & $0.191$ \\
    \midrule
    \multirow{4}{*}{HH\,help$\leftrightarrow$UF}
      & SFT & Small  & $68.4$ & $21.2$ \\
      & DPO & Small  & $25.2$ & $1.91$ \\
      & SFT & Medium & $60.5$ & $19.8$ \\
      & DPO & Medium & $17.7$ & $0.852$ \\
    \bottomrule
  \end{tabular}
\end{table}

\paragraph{Results.}
\Cref{tab:interference} reports the results. DPO suppresses cross-task interference by an order of magnitude or more in every configuration. Two qualitative patterns are notable: under DPO, both the Jaccard overlap and the interference decrease at the larger scale (e.g., $0.258 \to 0.191$ on the within-dataset pair and $1.91 \to 0.852$ on the cross-dataset pair), consistent with the sharper polarization of $A_i$ observed in \cref{sec:exp_data_sparsity}; and the same pattern holds for the cross-dataset pair (HH helpful $\leftrightarrow$ UltraFeedback), demonstrating that the effect generalizes beyond subsets of a single dataset to entirely independent preference data sources. These observations empirically support the framework of \cref{sec:wd_connection}: DPO's dual sparsity, combined with the small parameter-space overlap exhibited by semantically distinct preference tasks, yields cross-task inner products an order of magnitude smaller than those under SFT.

\paragraph{Model merging.}
To examine whether reduced interference carries over to merged models, we merge task vectors obtained by DPO or SFT on each task pair as $\theta_{\mathrm{merged}} = \theta_0 + \alpha (\tau_{\mathrm{target}} + \tau_{\mathrm{other}})$, $\alpha \in [0, 2]$ and measure the degradation $\Delta$ of each task's performance relative to its single-task model (\cref{app:merging}). As shown in \cref{tab:merging}, DPO yields a smaller $\Delta$ than SFT for both tasks in both task pairs, consistent with the reduced interference in \cref{tab:interference}.

\begin{table}[htbp]
  \centering
  \caption{Performance degradation $\Delta$ (single-task $-$ merged; lower is better) under task-vector merging on GPT-2 Medium. Task~1 is HH helpful; Task~2 is HH harmless or UltraFeedback. Full results are reported in \cref{app:merging}.}
  \label{tab:merging}
  \small
  \begin{tabular}{llcc}
    \toprule
    \textbf{Task Pair} & \textbf{Method} & \textbf{$\Delta$ Task~1} & \textbf{$\Delta$ Task~2} \\
    \midrule
    \multirow{2}{*}{HH help$\leftrightarrow$harm} & SFT & $0.110$ & $0.102$ \\
                                                  & DPO & $\mathbf{0.027}$ & $\mathbf{0.026}$ \\
    \midrule
    \multirow{2}{*}{HH help$\leftrightarrow$UF}   & SFT & $0.096$ & $0.030$ \\
                                                  & DPO & $\mathbf{0.036}$ & $\mathbf{0.014}$ \\
    \bottomrule
  \end{tabular}
\end{table}

%%%%%%%%%%%%%%%%%%%%%%%%%%%%%%%%%%%%%%%%%%%%%%%%%%%%%%%%%%%%

\section{Related Work}
\label{sec:related_work}

\paragraph{Optimization dynamics of DPO.}
Since \citet{rafailov2023direct} introduced DPO as a stable alternative to reinforcement learning from human feedback, several works have analyzed its dynamics theoretically. One line modifies the loss to address specific failure modes---IPO~\citep{azar2024general} avoids the Bradley--Terry assumption, and KTO~\citep{ethayarajh2024model} adopts a prospect-theoretic utility---but these are orthogonal to our analysis, which concerns the original objective. Closer to our approach, \citet{ren2025learning} linearize DPO under the NTK regime and use the resulting framework to characterize the \emph{squeezing effect}, a counter-intuitive decrease in the probability of preferred responses also studied by \citet{razin2025unintentional} under the term \emph{likelihood displacement}. We adopt the same NTK lens but extract a structurally distinct property: the \emph{direction} of the update in parameter space. Combining the classical implicit-bias result of \citet{soudry2018implicit} for separable logistic regression with the specific geometry of preference-pair gradients yields dual sparsity (\cref{thm:dual_sparsity}), which in turn provides a theoretical explanation for the small-subnetwork behavior observed empirically by \citet{mukherjee2025sparsity} on RL-based alignment. Likelihood displacement and dual sparsity are thus two complementary consequences of the same dynamics: one concerning the \emph{probabilities} the update changes, the other the \emph{parameters} it changes.

\paragraph{Task arithmetic, weight disentanglement, and model merging.}
\citet{ilharco2023editing} showed that linear arithmetic on task vectors compositionally steers model behavior, and \citet{ortizjimenez2023task} formalized weight disentanglement as a sufficient condition for this property, identifying via NTK tangent-space analysis that task vectors should act on disjoint regions of parameter space. A complementary line of work---DARE~\citep{yu2024language} and TIES-Merging~\citep{yadav2024tiesmerging}---engineers sparsity into task vectors as a \emph{post-hoc} intervention to suppress interference during merging. Our framework offers a different perspective: under DPO, sparsity emerges \emph{intrinsically} from the optimization dynamics, and the resulting parameter-space overlap between semantically distinct preference tasks is small enough that the disentanglement condition holds approximately without any explicit intervention (\cref{sec:exp_interference}).
Empirically, merging independently preference-optimized models has been reported to be competitive with or superior to data mixing, both for balancing safety and general-purpose capabilities~\citep{aakanksha2024mix} and for balancing helpfulness, honesty, and harmlessness~\citep{yang2025mix}.
Our analysis provides theoretical support for these empirical findings, attributing the effectiveness of merging DPO-trained models to the dual sparsity of their updates and the resulting small cross-task interference.
Our merging experiments (\cref{sec:exp_interference}) show that this translates into smaller performance degradation than SFT under task-vector merging; a larger-scale study of merging is left to future work.

%%%%%%%%%%%%%%%%%%%%%%%%%%%%%%%%%%%%%%%%%%%%%%%%%%%%%%%%%%%%

\section{Conclusion}
\label{sec:conclusion}

We have presented a theoretical analysis of why DPO induces sparse parameter updates. Under the NTK linearization, the logistic structure of the DPO loss drives the update direction toward a maximum-margin solution that exhibits dual sparsity: it is determined by a small number of support pairs, and the gradient feature vector of each pair is itself sparse due to gradient cancellation between preferred and dispreferred responses. Relating this structure to weight disentanglement, we showed empirically that DPO reduces cross-task interference by $11{\times}$ to $35{\times}$ relative to SFT and degrades less under model merging, and that both sources of sparsity hold from GPT-2 to 7--9B models.

\paragraph{Limitations.}
Our analysis relies on the NTK linearization, which is not guaranteed to hold throughout DPO training of large models. More fundamentally, parameter-level sparsity (\cref{asm:sparse_features}) is an empirically motivated assumption rather than a consequence derived from the architecture. It relies on token overlap between the preferred and dispreferred responses: for pairs whose responses are entirely disjoint (e.g., different topics or formats), the gradient cancellation does not occur, and this source of sparsity may weaken or vanish. A formal derivation of \cref{asm:sparse_features}, and a characterization of the degree of overlap it requires in practice, are left to future work.

%%%%%%%%%%%%%%%%%%%%%%%%%%%%%%%%%%%%%%%%%%%%%%%%%%%%%%%%%%%%

\bibliographystyle{plainnat} 
\bibliography{references}

%%%%%%%%%%%%%%%%%%%%%%%%%%%%%%%%%%%%%%%%%%%%%%%%%%%%%%%%%%%%

\appendix

\section{Proof and Numerical Evaluation of \texorpdfstring{\Cref{lem:data_sparsity}}{Lemma 2}}
\label{app:sv_bound}

\subsection{Proof}
\label{app:proof_sv_bound}

Throughout, we work with the frozen-NTK surrogate of the margin, $m_i^{\mathrm{lin}}(\theta) := m_i(\theta_0) + \langle \Phi_i, \theta - \theta_0 \rangle$, under which the DPO loss reduces exactly to the logistic regression \cref{eq:logistic_regression}; \cref{lem:data_sparsity} holds exactly for this model. Under $\pi_{\theta_0} = \pi_{\mathrm{ref}}$, the offsets vanish, $b_i = m_i(\theta_0) = 0$, so that $u^\star$ solves $\max_{\|u\| \le 1} \min_{i \in [n]} \langle u, \Phi_i \rangle$. By \cref{asm:separability}, the optimal value $\gamma_\star$ is positive, and hence the constraint is active: $\|u^\star\| = 1$.

\paragraph{KKT conditions.}
The problem is equivalent to the hard-margin problem $w^\star = \arg\min_w \frac{1}{2}\|w\|^2$ subject to $\langle w, \Phi_i \rangle \ge 1$ for all $i$, with $u^\star = w^\star / \|w^\star\|$ and $\gamma_\star = 1 / \|w^\star\|$. Its KKT conditions give $w^\star = \sum_i \hat{\alpha}_i \Phi_i$ with $\hat{\alpha}_i \ge 0$ and $\hat{\alpha}_i (\langle w^\star, \Phi_i \rangle - 1) = 0$. Thus $\alpha_i^\star = \hat{\alpha}_i / \|w^\star\|$ in \cref{eq:kkt}, and complementary slackness yields
\begin{equation}
  \langle u^\star, \Phi_i \rangle = \gamma_\star \qquad \text{for all } i \in \mathrm{SV}.
  \label{eq:kkt_active}
\end{equation}

\paragraph{Bound.}
Let $S := \mathrm{SV}$, $s := |S|$, and $\mathbf{1}_S \in \mathbb{R}^s$ the all-ones vector. Summing \cref{eq:kkt_active} over $S$ and applying the Cauchy--Schwarz inequality with $\|u^\star\| = 1$,
\begin{equation}
  s\, \gamma_\star = \Big\langle u^\star, \sum_{i \in S} \Phi_i \Big\rangle \le \Big\| \sum_{i \in S} \Phi_i \Big\| = \left( \mathbf{1}_S^\top K_{S,S}\, \mathbf{1}_S \right)^{1/2}.
\end{equation}
By the Rayleigh quotient, $\mathbf{1}_S^\top K_{S,S} \mathbf{1}_S \le \lambda_{\max}(K_{S,S})\, \|\mathbf{1}_S\|^2 = s\, \lambda_{\max}(K_{S,S})$. Combining the two, $s^2 \gamma_\star^2 \le s\, \lambda_{\max}(K_{S,S})$, i.e., $s\, \gamma_\star^2 \le \lambda_{\max}(K_{S,S})$. Since $K_{S,S}$ is a principal submatrix of the positive semidefinite matrix $K$, Cauchy's interlacing theorem gives $\lambda_{\max}(K_{S,S}) \le \lambda_{\max}(K)$. Finally, since $s$ is an integer with $s \le n$, we obtain $s \le \min\{ n, \lfloor \lambda_{\max}(K) / \gamma_\star^2 \rfloor \}$. \hfill$\qed$

\subsection{Numerical Evaluation}
\label{app:sv_numerical}

Computing $\gamma_\star$ exactly requires solving the max-margin problem. We instead evaluate the bound through the finite-time margin of the normalized trajectory, which is justified by the following corollary.

\begin{corollary}[Finite-time evaluation of the bound]
\label{cor:finite_time}
Under the assumptions of \cref{lem:data_sparsity}, let $\Delta\theta(T) \neq 0$ and define the finite-time margin
\begin{equation}
  \tilde{\gamma}_T := \min_{i \in [n]} \left\langle \frac{\Delta\theta(T)}{\|\Delta\theta(T)\|},\, \Phi_i \right\rangle.
\end{equation}
Then $\tilde{\gamma}_T \le \gamma_\star$. If, moreover, $\Delta\theta(T)$ follows the gradient flow of the frozen-NTK surrogate, then $\tilde{\gamma}_T \to \gamma_\star$ as $T \to \infty$. Consequently, whenever $\tilde{\gamma}_T > 0$,
\begin{equation}
  |\mathrm{SV}| \;\le\; \min\left\{ n,\ \left\lfloor \frac{\lambda_{\max}(K)}{\gamma_\star^2} \right\rfloor \right\} \;\le\; \min\left\{ n,\ \left\lfloor \frac{\lambda_{\max}(K)}{\tilde{\gamma}_T^2} \right\rfloor \right\}.
\end{equation}
\end{corollary}

\begin{proof}
Let $u_T := \Delta\theta(T)/\|\Delta\theta(T)\|$. Since $\|u_T\| = 1$, $u_T$ is feasible for the max-margin problem, and hence $\tilde{\gamma}_T = \min_{i} \langle u_T, \Phi_i \rangle \le \max_{\|u\| \le 1} \min_{i} \langle u, \Phi_i \rangle = \gamma_\star$. Under the frozen-NTK gradient flow, \cref{lem:implicit_bias} gives $u_T \to u^\star$; since $u \mapsto \min_{i} \langle u, \Phi_i \rangle$ is continuous as the minimum of finitely many continuous functions, $\tilde{\gamma}_T \to \min_{i} \langle u^\star, \Phi_i \rangle = \gamma_\star$. Finally, if $\tilde{\gamma}_T > 0$, then $1/\gamma_\star^2 \le 1/\tilde{\gamma}_T^2$, which, combined with \cref{eq:sv_bound}, $|\mathrm{SV}| \le n$, and the integrality of $|\mathrm{SV}|$, yields the chain of inequalities.
\end{proof}

The values reported below are computed from the actual training trajectory of the network rather than from the frozen-NTK gradient flow. For such a trajectory, the convergence statement of \cref{cor:finite_time} does not apply, but the inequality $\tilde{\gamma}_T \le \gamma_\star$ still holds, since it uses only $\|u_T\| = 1$; the reported ratio $\lambda_{\max}(K)/\tilde{\gamma}_T^2$ is therefore a conservative (upper) estimate of $\lambda_{\max}(K)/\gamma_\star^2$. \Cref{tab:sv_bound} reports this quantity in the setting of \cref{fig:data_sparsity} (GPT-2 Small/Medium, HH-RLHF, $n = 100$). For both models it is about $1/7$ to $1/9$ of $n$, guaranteeing that $|\mathrm{SV}|$ is roughly an order of magnitude smaller than $n$, consistent with the concentration observed in \cref{sec:exp_data_sparsity}.

\begin{table}[ht]
  \centering
  \caption{Conservative evaluation of the support-set bound \cref{eq:sv_bound} in the setting of \cref{fig:data_sparsity} ($n = 100$).}
  \label{tab:sv_bound}
  \small
  \begin{tabular}{lcc}
    \toprule
    \textbf{Model} & $\lambda_{\max}(K) / \tilde{\gamma}_T^2$ & \textbf{Implied bound on $|\mathrm{SV}|$} \\
    \midrule
    GPT-2 Small  & $14.58$ & $\le 14$ \\
    GPT-2 Medium & $11.33$ & $\le 11$ \\
    \bottomrule
  \end{tabular}
\end{table}

\section{Experimental Details}
\label{app:exp_details}

\paragraph{Time-integrated weights.}
The weights $A_i(T) = \int_0^T \sigma(-m_i(\theta(t)))\, dt$ are approximated by accumulating the discrete-time analogue $\sum_{t} \sigma(-m_i(\theta_t))$ over training steps. Since $A_i$ depends only on the margins, it requires no gradient features.

\paragraph{Experimental settings.}
\Cref{tab:exp_settings} summarizes the models, datasets, and numbers of preference pairs used in each experiment. The full fine-tuning DPO runs for the 7--8B models use up to eight NVIDIA A100 GPUs.

\begin{table}[ht]
  \centering
  \caption{Experimental settings. ``Pairs'' is the number of preference pairs on which each quantity is computed. HH: Anthropic HH-RLHF; UF: UltraFeedback Binarized; ep.: epochs.}
  \label{tab:exp_settings}
  \small
  \begin{tabular}{lllll}
    \toprule
    \textbf{Experiment} & \textbf{Models} & \textbf{Dataset} & \textbf{Pairs} & \textbf{Training} \\
    \midrule
    Data-level (\cref{fig:data_sparsity}) & GPT-2 S/M & HH & 100 & DPO, 1000 ep. \\
    Data-level (\cref{tab:data_sparsity}) & GPT-2 S/M & HH & 1000 & DPO, 1000 ep. \\
     & Llama-3-8B & UF & 1000 & Full-FT DPO, 3 ep. \\
     & Mistral-7B & UF & 1000 & Full-FT DPO, 3 ep. \\
    Parameter-level (\cref{fig:param_sparsity}) & GPT-2 M & HH & 100 & -- \\
    Parameter-level (\cref{tab:param_sparsity}) & All five & UF & 1000 & -- \\
    Interference (\cref{tab:interference}) & GPT-2 S/M & HH, UF & 100 & DPO / SFT \\
    Model merging (\cref{tab:merging}) & GPT-2 M & HH, UF & \cref{app:merging} & DPO / SFT \\
    \bottomrule
  \end{tabular}
\end{table}

\section{Evolution of Log-Probabilities During DPO Training}
\label{app:logps}

We record the log-probabilities of the preferred and dispreferred responses along the training trajectory that produces \cref{fig:data_sparsity} (GPT-2 Medium, HH-RLHF, $N = 100$). Each value is the sequence-level log-probability averaged over the 100 pairs, with $\text{logps/chosen} = \log \pi_\theta(y^+ \mid x)$ and $\text{logps/rejected} = \log \pi_\theta(y^- \mid x)$.

\begin{table}[ht]
  \centering
  \caption{Average log-probabilities of preferred (chosen) and dispreferred (rejected) responses during DPO training (GPT-2 Medium, HH-RLHF, $N = 100$).}
  \label{tab:logps}
  \small
  \begin{tabular}{rccc}
    \toprule
    \textbf{Epoch} & \textbf{logps/chosen} & \textbf{logps/rejected} & \textbf{Gap (chosen $-$ rejected)} \\
    \midrule
    1    & $-247.2$ & $-258.9$ & $11.7$ \\
    100  & $-280.5$ & $-334.5$ & $54.0$ \\
    1000 & $-308.3$ & $-436.9$ & $128.6$ \\
    \bottomrule
  \end{tabular}
\end{table}

As shown in \cref{tab:logps}, both log-probabilities decrease, but the rejected log-probability decreases considerably faster, so the gap widens from $11.7$ to $128.6$. Since $\pi_{\theta_0} = \pi_{\mathrm{ref}}$, the margin satisfies $m_i(\theta) = \beta\,(\text{gap}_i(\theta) - \text{gap}_i(\theta_0))$, where $\text{gap}_i$ denotes the per-pair log-probability gap; a widening gap thus directly increases the margin. As a result, $\sigma(-m_i) \to 0$ for easy pairs, which stops their contribution to $A_i$ and produces the polarization of $A_i$ observed in \cref{fig:data_sparsity}.

\section{Model Merging Experiments}
\label{app:merging}

\paragraph{Setup.}
For each task, we fine-tune GPT-2 Medium with DPO or SFT until convergence and take the task vector $\tau = \theta^\star - \theta_0$ from the best checkpoint. For each target task, we form $\theta_{\mathrm{merged}} = \theta_0 + \alpha\ (\tau_{\mathrm{target}} + \tau_{\mathrm{other}})$, where $\alpha \in \{0, 0.1, \ldots, 2.0\}$ is selected on a validation set, and report $\Delta = (\text{single-task}) - (\text{merged})$. We use the same two task pairs as in \cref{sec:exp_interference}.

\paragraph{Results.}
\Cref{tab:merging_full} reports the single-task and merged performance. On both task pairs, DPO degrades less than SFT under merging. The gap between the two methods is smaller than the order-of-magnitude reduction in the cross-task inner product (\cref{tab:interference}), since merged performance also depends on factors other than parameter-space overlap.

\begin{table}[ht]
  \centering
  \caption{Single-task and merged performance of GPT-2 under task-vector merging. $\Delta$ is the single-task minus merged performance (lower is better).}
  \label{tab:merging_full}
  \small
  \setlength{\tabcolsep}{4pt}
  \begin{tabular}{llcccccc}
    \toprule
    & & \multicolumn{2}{c}{\textbf{Single}} & \multicolumn{2}{c}{\textbf{Merged}} & \multicolumn{2}{c}{\textbf{$\Delta$}} \\
    \cmidrule(lr){3-4} \cmidrule(lr){5-6} \cmidrule(lr){7-8}
    \textbf{Task Pair} & \textbf{Method} & Task~1 & Task~2 & Task~1 & Task~2 & Task~1 & Task~2 \\
    \midrule
    \multirow{2}{*}{HH help$\leftrightarrow$harm} & SFT & $0.631$ & $0.625$ & $0.521$ & $0.523$ & $0.110$ & $0.102$ \\
                                                  & DPO & $0.668$ & $0.645$ & $0.641$ & $0.619$ & $\mathbf{0.027}$ & $\mathbf{0.026}$ \\
    \midrule
    \multirow{2}{*}{HH help$\leftrightarrow$UF}   & SFT & $0.631$ & $0.535$ & $0.518$ & $0.505$ & $0.096$ & $0.030$ \\
                                                  & DPO & $0.668$ & $0.572$ & $0.632$ & $0.558$ & $\mathbf{0.036}$ & $\mathbf{0.014}$ \\
    \bottomrule
  \end{tabular}
\end{table}

\section{Licenses}
\label{app:licenses}

\Cref{tab:licenses} lists the existing assets used in this paper together with their licenses and the versions (Hugging Face identifiers) we use.

\begin{table}[ht]
  \centering
  \caption{Existing assets, versions, and licenses.}
  \label{tab:licenses}
  \small
  \setlength{\tabcolsep}{4pt}
  \begin{tabular}{lll}
    \toprule
    \textbf{Asset} & \textbf{Version (Hugging Face)} & \textbf{License} \\
    \midrule
    GPT-2 Small              & \texttt{openai-community/gpt2}                 & Modified MIT \\
    GPT-2 Medium             & \texttt{openai-community/gpt2-medium}          & Modified MIT \\
    Llama-3-8B-Instruct      & \texttt{meta-llama/Meta-Llama-3-8B-Instruct}   & Llama 3 Community License \\
    Mistral-7B-Instruct-v0.2 & \texttt{mistralai/Mistral-7B-Instruct-v0.2}    & Apache 2.0 \\
    Gemma-2-9B-it            & \texttt{google/gemma-2-9b-it}                  & Gemma Terms of Use \\
    Anthropic HH-RLHF        & \texttt{Anthropic/hh-rlhf}                     & MIT \\
    UltraFeedback Binarized  & \texttt{HuggingFaceH4/ultrafeedback\_binarized} & MIT \\
    \bottomrule
  \end{tabular}
\end{table}

\end{document}